\documentclass[final,longnamesfirst]{colt2024}

\usepackage[normalem]{ulem}
\usepackage{amsmath,amsfonts,amssymb}
\usepackage{thmtools,thm-restate}
\usepackage{algorithm} 
\usepackage{algpseudocode}
\usepackage{mathtools}
\usepackage{cleveref}
\usepackage{bbm}
\usepackage{parskip}

\usepackage{tikz}
\usetikzlibrary{shapes.geometric,calc}

\DeclareRobustCommand{\eg}{e.g.,\@\xspace}

\newcommand{\bool}{\{0,1\}}

\newcommand{\R}{\mathbb{R}}
\newcommand{\E}{\mathbb{E}}

\newcommand{\pr}{\mathbb{P}}
\renewcommand{\Pr}{\pr}

\newcommand{\scA}{\mathcal{A}}

\newcommand{\scF}{\mathcal{F}}

\newcommand{\scH}{\mathcal{H}}
\newcommand{\scL}{\mathcal{L}}
\newcommand{\scO}{\mathcal{O}}
\newcommand{\scP}{\mathcal{P}}

\newcommand{\scT}{\mathcal{T}}
\newcommand{\scV}{\mathcal{V}}

\newcommand{\scX}{\mathcal{X}}

\newcommand{\eps}{\varepsilon}

\DeclareMathOperator{\depth}{depth}

\DeclareMathOperator{\vc}{VC}

\newcommand{\mybox}[1]{%
\par\medskip\noindent%
\framebox[\textwidth][c]{#1}
\par\medskip }

\renewcommand{\tilde}{\widetilde}
\DeclareMathOperator{\Alg}{Alg} %algebra
\DeclareMathOperator{\supp}{supp} %support
\DeclareMathOperator{\clos}{clos}
\DeclareMathOperator{\clospw}{clos_{pw}} % pointwise closure

\renewcommand{\epsilon}{\varepsilon}

\renewcommand{\scA}{\mathcal{A}} %example algebra
\newcommand{\NN}{\mathbb{N}} %natural numbers
\newcommand{\N}{\NN} %natural numbers
\newcommand{\RR}{\mathbb{R}} %real numbers
 
\newcommand{\Trees}[3]{\scT^{#1}_{#2}(#3)}
\newcommand{\AllTrees}{\scT_{\scH}}
\newcommand{\Depth}[3]{\depth^{#1}_{#2}\!\left(#3\right)}
\newcommand{\GCM}[3]{\Gamma^{#1}_{#2}\!\left(#3\right)}
\newcommand{\Pd}{P}
\newcommand{\PD}{\scP}

\newcommand\restr[2]{{% we make the whole thing an ordinary symbol
  \left.\kern-\nulldelimiterspace % automatically resize the bar with \right
  #1 % the function
  \vphantom{\big|} % pretend it's a little taller at normal size
  \right|_{#2} % this is the delimiter
  }}

\providecommand{\proofsketchname}{Proof sketch}
{%
\par\noindent{\bfseries\upshape \proofsketchname\ }%
}%
{\jmlrQED}

\usepackage[shortlabels]{enumitem}
\usepackage{cleveref}

\title{A Theory of Interpretable Approximations}

\usepackage{times}
\coltauthor{%
 \Name{Marco Bressan} \Email{marco.bressan@unimi.it}\\
 \addr Università degli Studi di Milano, Italy
 \AND
 \Name{Nicolò Cesa-Bianchi} \Email{nicolo.cesa-bianchi@unimi.it}\\
 \addr Università degli Studi di Milano, Italy\\
 \addr Politecnico di Milano, Italy
 \AND
 \Name{Emmanuel Esposito} \Email{emmanuel@emmanuelesposito.it}\\
 \addr Università degli Studi di Milano, Italy\\
 \addr Istituto Italiano di Tecnologia, Italy
 \AND
 \Name{Yishay Mansour} \Email{mansour.yishay@gmail.com}\\
 \addr Tel Aviv University, Israel\\
 \addr Google Research
 \AND
 \Name{Shay Moran} \Email{smoran@technion.ac.il}\\
 \addr Technion, Israel\\
 \addr Google Research
 \AND
 \Name{Maximilian Thiessen} \Email{maximilian.thiessen@tuwien.ac.at}\\
 \addr TU Wien, Austria%
}

\begin{document}

\maketitle

\begin{abstract}%
Can a deep neural network be approximated by a small decision tree based on simple features? This question and its variants are behind the growing demand for machine learning models that are \emph{interpretable} by humans. In this work we study such questions by introducing \emph{interpretable approximations}, a notion that captures the idea of approximating a target concept $c$ by a small aggregation of concepts from some base class $\scH$. In particular, we consider the approximation of a binary concept $c$ by decision trees based on a simple class $\scH$ (e.g., of bounded VC dimension), and use the tree depth as a measure of complexity.
Our primary contribution is the following remarkable trichotomy. For any given pair of $\scH$ and $c$, exactly one of these cases holds: (i) $c$ cannot be approximated by~$\scH$ with arbitrary accuracy; (ii) $c$ can be approximated by $\scH$ with arbitrary accuracy,
but there exists no universal rate that bounds the complexity of the approximations as a function of the accuracy;
or (iii) there exists a constant $\kappa$ that depends only on~$\scH$ and $c$ such that, for \emph{any} data distribution and \emph{any} desired accuracy level, $c$ can be approximated by $\scH$ with a complexity not exceeding $\kappa$.
This taxonomy stands in stark contrast to the landscape of supervised classification, which offers a complex array of distribution-free and universally learnable scenarios. We show that, in the case of interpretable approximations, even a slightly nontrivial a-priori guarantee on the complexity of approximations implies approximations with constant (distribution-free and accuracy-free) complexity. We extend our trichotomy to classes $\scH$ of unbounded VC dimension and give characterizations of interpretability based on the algebra generated by $\scH$.
\end{abstract}

\begin{keywords}%
interpretability, learning theory, boosting
\end{keywords}

\section{Introduction}
Many machine learning techniques, such as deep neural networks, produce large and complex models whose inner workings are difficult to grasp. In sectors such as healthcare and law enforcement, where the stakes of automated decisions are high, this is a serious problem: complex models make it hard to explain the rationale behind an outcome, or why two similar inputs produce different outcomes.
In those cases, \emph{interpretable} models may become the preferred choice.
Although there is an ongoing debate around the notion of interpretability \citep{erasmus2021interpretability}, decision trees are typically considered as the quintessential example of interpretable models \citep{molnar2022}: ones that favor a transparent decision-making process, and that allow users to understand how individual features influence predictions.
A line of research in this area studies the extent to which small decision trees can approximate some specific learning models, such as neural networks~\citep{craven1995extracting} and $k$-means classifiers~\citep{dasgupta2020explainable}.
Inspired by these results, we develop a general theory of interpretability viewed as approximability via simple decision trees. Our guiding principle can be summarized as follows.
\mybox{
	\textbf{Interpretable approximations = Small aggregations of simple hypotheses.}
}
In analogy with PAC learning, we focus on binary classification tasks and view a classifier (e.g., a neural network) as a concept $c \subseteq X$, where $X$ is the data domain. Now let $\scH \subseteq 2^X$ be a family of simple hypotheses, for instance decision stumps or halfspaces. Our goal is to understand how well~$c$ can be approximated by aggregating a small set of elements in $\scH$.
To formalize this goal in the language of decision trees we introduce two notions. First, we say that $c$ is \emph{approximable} by $\scH$ if, under any given data distribution, there exists a finite decision tree using splitting functions from $\scH$ that approximates $c$ arbitrarily well. Moreover, if the approximation can be always achieved using a shallow tree, we say that $c$ is \emph{interpretable} by $\scH$.
It is easy to see that, depending on $c$ and $\scH$, one may have interpretability, approximability but not interpretability, or even non-approximability. In \Cref{sec:definitions}, we give explicit examples of pairs $(c,\scH)$ for each one of the three above cases.

Note that in this initial study on the general structure of interpretable approximations we focus on the fundamental question of what conditions ensure the existence of accurate approximations and interpretations. Important topics, such as the informational or computational complexity of obtaining accurate interpretations, are not addressed in this work. 
Note also that we do not make any specific assumption on the data distribution $\Pd$. Our approach is thus in line with standard notions and theories in machine learning---e.g., universal Bayes consistency \citep{devroye2013probabilistic}, PAC learnability \citep{shalev2014understanding}, and universal learnability ~\citep{bousquet2021theory}---as it encompasses both distribution-free and distribution-dependent guarantees.

While our primary focus is not algorithmic, our work reveals profound connections within the algorithmic framework of boosting. Indeed, there is a clear relationship between boosting, which involves the aggregation of weak hypotheses to learn a target concept, and interpretable approximation, which concerns the aggregation of simple hypotheses to approximate a target concept. However, our work uncovers and exploits deeper links at a technical level. In particular, our general construction that gives decision trees whose depth depends logarithmically on the accuracy is based on boosting decision trees, and its analysis uses potential functions from this line of work. Our improved bound for the VC classes, which provides approximating decision trees with constant (accuracy- and distribution-free) depth, is somewhat more subtle; it is also based on a boosting perspective, this time using majority-vote based algorithms and the minimax theorem. However, to eliminate the dependency on the accuracy, we utilize tools from VC theory, particularly uniform convergence.

\subsection{Contributions}

\noindent\textbf{Degrees of interpretability (Section~\ref{sec:definitions}).}
We introduce our learning-theoretic notions of approximability and interpretability. Informally speaking, we use the depth of the shallowest approximating tree to measure the extent to which a certain concept $c$ is interpretable by a given class $\scH$ (e.g., hyperplanes or single features). Approximability is our weakest notion, as we do not constrain the rate at which the rate of the shallowest approximating tree grows as a function of the desired accuracy. Our strongest notion is instead interpretability with a tree depth that is constant with respect to both accuracy and data distribution. In between these two extremes, a wide variety of behaviors is possible, as the tree depth may grow at different rates that may be uniform, or depend on the data distribution (similarly to the distinction between PAC learning and universal learning).

\noindent\textbf{Collapse of the degrees (Section~\ref{sec:trichotomy}).}
We prove that the range of possible behaviors collapses dramatically, and only three cases are actually possible: $c$ is uniformly interpretable by $\scH$, $c$ is approximable but not interpretable by $\scH$, $c$ is not approximable by $\scH$. If the class $\scH$ of splits has bounded VC dimension, which conforms to our request that $\scH$ be simple, we show that whenever $c$ is interpretable (possibly with a distribution-dependent rate) then it is uniformly interpretable by $\scH$ \emph{at constant depth}. This means that, for every data distribution $\Pd$ and every accuracy $\eps > 0$, there exists an $\scH$-based decision tree that approximates $c$ with accuracy $\eps$ and whose depth is bounded \emph{by a constant} depending only on $c,\scH$ (but not on $\Pd,\eps$). 
Thus, whenever $c$ is interpretable at some arbitrary rate, it is in fact interpretable at a constant rate. 
We show a similar collapse for classes $\scH$ of unbounded VC dimension: in this case, we show that interpretability collapses to uniform interpretability at logarithmic depth $\scO\bigl(\log \frac1\eps\bigr)$.

\noindent\textbf{Algebraic characterizations (\Cref{sec:algebraic}).}
We prove that the trichotomy described above can be characterized in terms of algebras and closures over $\scH$. For example, we show that if $\scH$ has bounded VC dimension, then $c$ is interpretable at constant depth if and only if $c$ is in the algebra generated by the \emph{closure} of $\scH$, i.e., the family of all the concepts that can be approximated arbitrarily well by single hypotheses of $\scH$. We also present a simpler characterization when the domain $X$ is countable.

\noindent\textbf{Extension to other complexity measures (\Cref{sec:gen-rep}).}
Finally, we exploit the equivalence between $\scH$-based decision trees and Boolean formulae over $\scH$ to show that the trichotomy above holds for a large class of complexity measures, including not only tree-depth but also, for example, circuit size. In particular, we show that for any complexity measure in our class, interpretability collapses to uniform interpretability at constant complexity rate for VC classes and at polynomial complexity rate for non-VC classes.

\section{Related Work}
According to \citet{molnar2022}, there are different approaches to interpretability in learning. One important distinction is between local explanation, where we explain the prediction of the model on a single data point, and global interpretation, where we explain the model itself. In this work we focus on the latter. A common approach to global interpretation is to use simpler ``interpretable'' models (e.g., decision trees) to approximate more complex ones \citep{craven1995extracting}. This is known as \emph{post-hoc interpretability} \citep{molnar2022}.
For example, \citet{zhang2019interpreting} used decision trees to interpret convolutional neural networks.
Formally, interpretability can be modeled as a property of a classifier. For example, \citet{dziugaite2020enforcing} define a variant of empirical risk minimization (ERM), where each classifier in a given class $\scH$ is either interpretable or not, and the task is to learn an interpretable one even though the target concept is not necessarily interpretable. We generalize this setup by assigning a complexity measure to each classifier, e.g., the depth for decision trees. This allows to trade-off the desired accuracy $\eps$ and the maximum depth of a decision tree one is willing to call interpretable. 
Learning-theoretic perspectives on interpretability are rare and typically not covered in standard books and surveys. One important line of work initiated by \citet{dasgupta2020explainable} deals with the problem of approximating a given $k$-means or $k$-median clustering with decision trees. From this perspective, our setup can be seen as a generalization from clusterings to arbitrary concepts. However, that line of work focuses on efficient algorithms to compute decision trees with $k$ leaves and approximation guarantees in terms of the $k$-means or $k$-medians cost function, not in terms of classification error under a distribution as we do here.
\citet{bastani2017interpretability} discuss a related problem setup where a given classifier is approximated using a decision tree. Under strong assumptions, the authors state convergence results for the proposed decision tree. However, they do not state bounds on the required depth which is assumed to be given as a hyperparameter.
Some algorithmic analyses exist for specific cases of hypothesis spaces and standard explainers. For example, \citet{garreau2020explaining} analyse LIME \citep{ribeiro2016should}, one of the most used explanation techniques. \citet{li2021a} discuss generalization bounds for local explainers. \citet{blanc2021provably} introduce a local variant of our setup with the goal of explaining the classification $f(x)$ of a single instance $x$ using a conjunction with small size (i.e., a small decision list). Their results cannot be used for our goal of global interpretation as one would have to take the union of all the local conjunctions for all (potentially infinite) instances $x$. Closer to our setup, \citet{moshkovitz2021connecting} state bounds on the depth of a decision tree required to fit a linear classifier with margin. Similarly to us, they also strongly rely on boosting arguments.
\citet{vidal2020born} give upper bounds on the number of nodes  of a single decision tree to approximate an ensemble of trees.
While mainly focusing on local explainability, \citet{blanc2021provably} also state bounds on the depth of a decision tree required to fit an arbitrary classifier $f\colon \bool^d\to \bool$ under the uniform distribution on $\bool^d$. % for some fixed dimension $d$.
They do so by relying on classical bounds on the depth in terms of certificate complexity \citep{smyth2002reimer,tardos1989query}. As we focus instead on general hypothesis classes and distributions, their results are not directly comparable to ours. 
\section{Preliminaries and Notation}
Let $X$ be any domain. We denote by $\Pd$ a distribution\footnote{By default we assume a fixed but otherwise arbitrary $\sigma$-algebra on $X$ and that all functions/sets discussed in our theorems are measurable. We also borrow standard assumptions on the underlying $\sigma$-algebra which allow us to use the VC Theorem~\citep{vc1971}. See, e.g.,~\citet{BlumerEHW89}.} on $X$ and by $\PD(X)$ the set of all distributions on $X$, by $\scH\subseteq 2^X$ a hypothesis class on $X$, and by $\vc(X,\scH)$ its VC dimension. We denote by $\Alg(\scH)$ the algebra generated by $\scH$, i.e., the smallest set system $\scA\subseteq 2^X$ closed under complements and finite unions such that $\scH\subseteq\scA$ and $\emptyset,X\in \scA$. The $\sigma$-algebra $\sigma(\scH)$ is the smallest algebra containing $\scH$ that is closed under countable unions.
We denote by $c \in 2^X$ an arbitrary concept (not necessarily in $\scH$). As usual we also view $c$ as a binary classification function $c\colon X\to\bool$. Our goal is to understand how well $c$ can be approximated using aggregations of hypotheses in $\scH$.
We let $\NN$ denote the naturals including $0$, and $\NN^+ = \NN \setminus \{0\}$.

A decision tree over $X$ is a full finite binary tree $T$ with nodes $\scV(T)$, where every leaf $z \in \scL(T)$ holds a label $\ell_z \in \bool$ and every internal node $v \in \scV(T)\setminus \scL(T)$ holds a \emph{decision stump} $f_v \colon X \to \bool$. The depth (or height) of $T$ is denoted as $\depth(T)$. We say $T$ is \emph{$\scH$-based} if $f_v \in \scH$ for all $v \in \scV(T)$, and we denote by $\AllTrees$ the set of all $\scH$-based decision trees.  We also use $T$ to denote the binary classifier $T\colon X \to \bool$ induced by $T$ in the standard way. Note that $\AllTrees \equiv \Alg(\scH)$, as any $\scH$-based tree $T$ can be rewritten as a Boolean formula and vice versa. For every $d\in\NN_+$ we let $\Alg_d(\scH) = \{T \in \Alg(\scH) \,:\, \depth(T) \le d \}$.
Given $\Pd \in \PD(X)$ and a concept $c \in 2^X$, the \emph{loss} of $T$ with respect to $c$ under $\Pd$ is $L_{\Pd}(T,c) = \Pr_{x\sim\Pd}(T(x) \neq c(x)) = \Pd\bigl(T^{-1}(1) \triangle c\bigr)$, where $A \triangle B = (A\setminus B) \cup (B\setminus A)$ is the symmetric difference between $A$ and $B$.

An $\eps$-accurate $\scH$-approximation of $c$ under $\Pd$ is an $\scH$-based decision tree $T$ with $L_{\Pd}(T,c) \le \eps$. The set of all such trees is denoted as $\Trees{c}{\scH}{\eps \mid \Pd}$, which is also known as the $\eps$-Rashomon set \citep{fisher2019all}, and their minimal depth is
\begin{equation}
\Depth{c}{\scH}{\eps \mid \Pd} = \inf_{T\in\Trees{c}{\scH}{\eps \mid \Pd}}\depth(T) \enspace.
\end{equation}

\section{Approximability and Interpretability}
\label{sec:definitions}
This section introduces the key definitions used in our results. We start with the definition of approximability.
\begin{definition}[Approximability]
\label{def:approx}
A concept $c$ is approximable by $\scH$ if $\Trees{c}{\scH}{\eps\mid\Pd}\neq\emptyset$ %\Depth{c}{\scH}{\eps \mid \Pd} < \infty$
for every distribution $\Pd\in\PD(X)$ and every $\eps>0$.
\end{definition}
Approximability is our weakest notion, as it only requires that for any desired accuracy value a tree approximating $c$ exists under any distribution, without any constraint on its depth. In fact, there may not even exist a function $f$ such that $\Depth{c}{\scH}{\eps \mid \Pd}$ is bounded by $f(\eps)$ for all distributions $\Pd$.

For example, for $X=\R^d$ let 
$c$ be the unit $d$-dimensional Euclidean ball centered at the origin
and $\scH$
be the family of affine halfspaces whose boundary is orthogonal to, say, the $d$-th dimension. Then, any finite aggregation $T$ of such halfspaces is unable to discern points that are aligned along the $i$-th dimension for any $i \ne d$ and, thus, necessarily incurs a constant $L_{\Pd}(T,c)$ for some distribution~$\Pd$.
On the other hand, if we extend $\scH$ to be the family of all halfspaces in $X=\R^d$, then it is possible to show that we can approximate the unit ball $c$
up to any accuracy under any distribution.
Indeed, it is known that a variant of the $1$-nearest neighbour ($1$-NN) algorithm is universally strongly Bayes consistent in essentially separable metric spaces \citep{hanneke2021universal}, and any $1$-NN classifier corresponds to a finite Voronoi partition which can be represented as an $\scH$-based decision tree.
However, we expect the number of Voronoi cells, and thus the depth of the $\scH$-based decision tree representing it, to grow larger as the distribution $\Pd$ concentrates around the decision boundary (that is, the surface of the unit ball $c$).
Consider, for instance, the family of distributions $\Pd_\alpha$ with $\alpha\in(0,1)$, where each $\Pd_\alpha$ has support corresponding to the spherical shell $B^d(1+\alpha)\setminus B^d(1-\alpha)$ with inner radius $1-\alpha$ and outer radius $1+\alpha$ (here we denote by $B^d(r) = \{x\in\R^d : \|x\|_2 \le r\}$ the origin-centered Euclidean ball of radius $r > 0$ in $\R^d$). Then, we expect the number of Voronoi cells defining the decision boundary of the $1$-NN classifiers that guarantee loss at most $0<\eps\le 1$ to grow as $\alpha \to 0^+$.
Figures~\ref{fig:approx-ball-1} and~\ref{fig:approx-ball-2} illustrate these examples in $\R^2$.

Next, we define interpretability. Recall that we view an interpretation as an approximation via a tree of small depth. We formalize ``small'' by requiring the existence of a function that bounds the depth of the tree in terms of its accuracy.
\begin{definition}[Interpretability]
\label{def:interp}
A concept \(c\) is \emph{interpretable} by $\scH$ if there is a function $f \colon (0,1] \to\NN$ such that, for every distribution $\Pd\in\PD(X)$, there exists $\eps_\Pd>0$ for which 
\[
\Depth{c}{\scH}{\eps \mid \Pd} \le f(\eps) \qquad \text{for all} \quad 0 < \eps \le \eps_\Pd \enspace.
\]
If this is the case, then we say that $c$ is interpretable by $\scH$ at depth rate $f$.
\\
A concept \(c\) is \emph{uniformly interpretable} by $\scH$ if there is a function $f' \colon (0,1] \to\NN$ such that
\[
\Depth{c}{\scH}{\eps \mid \Pd} \le f'(\eps) \qquad \text{for all} \quad \Pd\in\PD(X) \quad\text{and}\quad 0 < \eps \le 1 \enspace.
\]
If this is the case, then we say that $c$ is uniformly interpretable by $\scH$ at depth rate $f'$.
\end{definition}
Note that interpretability requires the bound on the depth to hold only for values of $\eps$ that are smaller than a certain threshold $\eps_{\Pd}$ which may depend on the distribution $\Pd$. Uniform interpretability, instead, requires the depth bound to hold for all $\eps$ irrespective of the distribution.

Recalling the above example with the Euclidean space $X=\R^d$ as domain and the family of halfspaces as the hypothesis class $\scH$, if the concept $c$ corresponds to the unit Euclidean ball with margin $\mu>0$ then $c$ is uniformly interpretable.
More formally, such a concept $c$ can be modeled as a partial function $c\colon X\to\bool$ with natural domain $\tilde X\subset X$, where points in the margin belong to $X\setminus\tilde X = B^d(1+\mu)\setminus B^d(1)$ and $c^{-1}(1)=B^d(1)$.
Then, without loss of generality, the same definitions and results apply as if the concept $c$ was a total function by restricting the domain to $\tilde X$ and, for every distribution $\Pd\in\PD(X)$, considering the distribution $\tilde{\Pd}(\cdot) = \Pd(\cdot \mid \tilde X)$ instead. This follows from the fact that we incur no mistakes for any labeling of points that do not belong to the domain of the ``partial'' concept $c$, and the loss of any $\scH$-approximation $T$ of $c$ is thus given by $L_{\tilde{\Pd}}(T,c)$.
By reusing geometric results on the approximation of convex bodies, there exists a polytope $Q$ such that $B^d(1) \subseteq Q \subseteq B^d(1+\mu)$, whose (finite) number of vertices is bounded from above by a function of $d$ and $\mu$ \citep{naszodi2019approximating}. The polytope $Q$ thus separates the positively labeled points $B^d(1)$ from the negatively labeled ones---achieving loss $0$ under any distribution with support $\tilde X$---and it is equivalently representable as an $\scH$-based decision tree with depth bounded by a function of $d$ and $\mu$ only (i.e., the intersection of halfspaces associated to the facets of $Q$).
See Figure~\ref{fig:approx-ball-3} for an illustration of this example in $\R^2$.

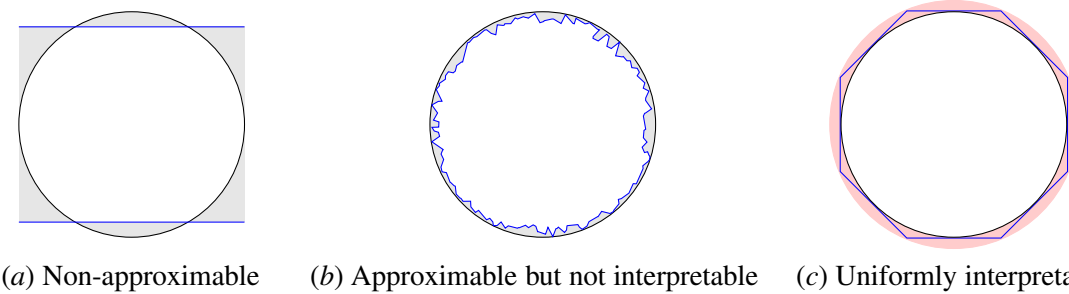
\begin{figure}[ht]
    \centering
    \subfigure[Non-approximable]{%
        \hspace{4mm}\begin{tikzpicture}[scale=3]
    \tikzstyle{mypath}=[draw,color=blue];

    \node (c) at (0.5,0.5) {};
    \node (x1) at (0,0.933) {};
    \node (x2) at (1,0.933) {};
    \node (x3) at (1,0.067) {};
    \node (x4) at (0,0.067) {};
    \filldraw[color=white!0] (c) circle (0.55);

    \fill[color=black!10] (x1.center) -- (x2.center) -- (x3.center) -- (x4.center) -- cycle;
    \filldraw[color=white] (c) circle (0.5);
    \filldraw[fill=black!10!white,draw=none] ($(c) + (120:0.5)$) arc (120:60:0.5);
    \filldraw[fill=black!10!white,draw=none] ($(c) + (240:0.5)$) arc (240:300:0.5);
    \draw (c) circle (0.5);
    \draw[mypath] (x1.center) -- (x2.center);
    \draw[mypath] (x3.center) -- (x4.center);
\end{tikzpicture}\hspace{2mm}\label{fig:approx-ball-1}
    }
\hfill
    \subfigure[Approximable but not interpretable]{%
        \hspace{16mm}\begin{tikzpicture}[scale=3]
    \filldraw[color=white!0] (0.5, 0.5) circle (0.55);
    \tikzstyle{mypath}=[draw,color=blue];
    
    \node (x1) at (0.2713, 0.9041) {};
    \node (x2) at (0.2513, 0.9200) {};
    \node (x3) at (0.2423, 0.9186) {};
    \node (x4) at (0.2080, 0.8896) {};
    \node (x5) at (0.1925, 0.8849) {};
    \node (x6) at (0.1744, 0.8644) {};
    \node (x7) at (0.1724, 0.8643) {};
    \node (x8) at (0.1666, 0.8621) {};
    \node (x9) at (0.1448, 0.7771) {};
    \node (x10) at (0.1402, 0.7772) {};
    \node (x11) at (0.1186, 0.7615) {};
    \node (x12) at (0.1122, 0.7366) {};
    \node (x13) at (0.1031, 0.7291) {};
    \node (x14) at (0.0818, 0.7337) {};
    \node (x15) at (0.0691, 0.7167) {};
    \node (x16) at (0.0684, 0.7006) {};
    \node (x17) at (0.0405, 0.6814) {};
    \node (x18) at (0.0429, 0.6748) {};
    \node (x19) at (0.0411, 0.6705) {};
    \node (x20) at (0.0280, 0.6553) {};
    \node (x21) at (0.0534, 0.6033) {};
    \node (x22) at (0.0369, 0.6024) {};
    \node (x23) at (0.0085, 0.5865) {};
    \node (x24) at (0.0337, 0.5510) {};
    \node (x25) at (0.0157, 0.5193) {};
    \node (x26) at (0.0399, 0.5114) {};
    \node (x27) at (0.0387, 0.4884) {};
    \node (x28) at (0.0089, 0.4873) {};
    \node (x29) at (0.0189, 0.4712) {};
    \node (x30) at (0.0376, 0.4654) {};
    \node (x31) at (0.0428, 0.4474) {};
    \node (x32) at (0.0085, 0.4234) {};
    \node (x33) at (0.0586, 0.3889) {};
    \node (x34) at (0.0539, 0.3809) {};
    \node (x35) at (0.0493, 0.3688) {};
    \node (x36) at (0.0515, 0.3643) {};
    \node (x37) at (0.0608, 0.3448) {};
    \node (x38) at (0.0487, 0.3105) {};
    \node (x39) at (0.0497, 0.3080) {};
    \node (x40) at (0.0536, 0.2902) {};
    \node (x41) at (0.0730, 0.2864) {};
    \node (x42) at (0.0892, 0.2438) {};
    \node (x43) at (0.0916, 0.2431) {};
    \node (x44) at (0.0919, 0.2427) {};
    \node (x45) at (0.1069, 0.2285) {};
    \node (x46) at (0.1124, 0.2262) {};
    \node (x47) at (0.1163, 0.1911) {};
    \node (x48) at (0.1335, 0.1734) {};
    \node (x49) at (0.1365, 0.1730) {};
    \node (x50) at (0.1382, 0.1714) {};
    \node (x51) at (0.1467, 0.1593) {};
    \node (x52) at (0.1801, 0.1597) {};
    \node (x53) at (0.1814, 0.1346) {};
    \node (x54) at (0.2014, 0.1099) {};
    \node (x55) at (0.2154, 0.1191) {};
    \node (x56) at (0.2508, 0.0773) {};
    \node (x57) at (0.2652, 0.0839) {};
    \node (x58) at (0.2787, 0.0678) {};
    \node (x59) at (0.2808, 0.0517) {};
    \node (x60) at (0.2975, 0.0517) {};
    \node (x61) at (0.3153, 0.0677) {};
    \node (x62) at (0.3338, 0.0401) {};
    \node (x63) at (0.3470, 0.0505) {};
    \node (x64) at (0.3878, 0.0400) {};
    \node (x65) at (0.3919, 0.0427) {};
    \node (x66) at (0.3955, 0.0420) {};
    \node (x67) at (0.4234, 0.0531) {};
    \node (x68) at (0.4367, 0.0426) {};
    \node (x69) at (0.4575, 0.0548) {};
    \node (x70) at (0.4693, 0.0483) {};
    \node (x71) at (0.4773, 0.0246) {};
    \node (x72) at (0.4785, 0.0250) {};
    \node (x73) at (0.5122, 0.0134) {};
    \node (x74) at (0.5230, 0.0441) {};
    \node (x75) at (0.5446, 0.0030) {};
    \node (x76) at (0.5580, 0.0375) {};
    \node (x77) at (0.5636, 0.0392) {};
    \node (x78) at (0.5689, 0.0447) {};
    \node (x79) at (0.5965, 0.0212) {};
    \node (x80) at (0.5965, 0.0212) {};
    \node (x81) at (0.6053, 0.0320) {};
    \node (x82) at (0.6072, 0.0501) {};
    \node (x83) at (0.6089, 0.0513) {};
    \node (x84) at (0.6555, 0.0560) {};
    \node (x85) at (0.6624, 0.0447) {};
    \node (x86) at (0.6749, 0.0349) {};
    \node (x87) at (0.6961, 0.0496) {};
    \node (x88) at (0.7048, 0.0690) {};
    \node (x89) at (0.7436, 0.0916) {};
    \node (x90) at (0.7451, 0.0914) {};
    \node (x91) at (0.7608, 0.1146) {};
    \node (x92) at (0.7983, 0.1052) {};
    \node (x93) at (0.8131, 0.1139) {};
    \node (x94) at (0.8145, 0.1194) {};
    \node (x95) at (0.8378, 0.1481) {};
    \node (x96) at (0.8349, 0.1677) {};
    \node (x97) at (0.8480, 0.1712) {};
    \node (x98) at (0.8819, 0.1957) {};
    \node (x99) at (0.9006, 0.2372) {};
    \node (x100) at (0.9031, 0.2378) {};
    \node (x101) at (0.9098, 0.2367) {};
    \node (x102) at (0.9135, 0.2371) {};
    \node (x103) at (0.9145, 0.2373) {};
    \node (x104) at (0.9151, 0.2378) {};
    \node (x105) at (0.9445, 0.3043) {};
    \node (x106) at (0.9432, 0.3166) {};
    \node (x107) at (0.9633, 0.3333) {};
    \node (x108) at (0.9644, 0.3333) {};
    \node (x109) at (0.9724, 0.3430) {};
    \node (x110) at (0.9751, 0.3569) {};
    \node (x111) at (0.9700, 0.3648) {};
    \node (x112) at (0.9645, 0.3791) {};
    \node (x113) at (0.9508, 0.3755) {};
    \node (x114) at (0.9404, 0.3993) {};
    \node (x115) at (0.9619, 0.4219) {};
    \node (x116) at (0.9457, 0.4411) {};
    \node (x117) at (0.9485, 0.4738) {};
    \node (x118) at (0.9485, 0.4738) {};
    \node (x119) at (0.9561, 0.4790) {};
    \node (x120) at (0.9458, 0.5373) {};
    \node (x121) at (0.9514, 0.5495) {};
    \node (x122) at (0.9550, 0.5506) {};
    \node (x123) at (0.9821, 0.5694) {};
    \node (x124) at (0.9556, 0.6068) {};
    \node (x125) at (0.9561, 0.6176) {};
    \node (x126) at (0.9702, 0.6482) {};
    \node (x127) at (0.9453, 0.6611) {};
    \node (x128) at (0.9523, 0.6844) {};
    \node (x129) at (0.9182, 0.7160) {};
    \node (x130) at (0.9281, 0.7316) {};
    \node (x131) at (0.9185, 0.7491) {};
    \node (x132) at (0.8945, 0.7520) {};
    \node (x133) at (0.8883, 0.7819) {};
    \node (x134) at (0.8878, 0.7824) {};
    \node (x135) at (0.8729, 0.7827) {};
    \node (x136) at (0.8567, 0.8016) {};
    \node (x137) at (0.8519, 0.7996) {};
    \node (x138) at (0.8228, 0.8045) {};
    \node (x139) at (0.8402, 0.8653) {};
    \node (x140) at (0.8036, 0.8369) {};
    \node (x141) at (0.7840, 0.8467) {};
    \node (x142) at (0.7944, 0.8890) {};
    \node (x143) at (0.7835, 0.9001) {};
    \node (x144) at (0.7455, 0.8704) {};
    \node (x145) at (0.7225, 0.8820) {};
    \node (x146) at (0.7234, 0.8904) {};
    \node (x147) at (0.7510, 0.9195) {};
    \node (x148) at (0.7422, 0.9226) {};
    \node (x149) at (0.7094, 0.9501) {};
    \node (x150) at (0.7056, 0.9509) {};
    \node (x151) at (0.7029, 0.9500) {};
    \node (x152) at (0.6963, 0.9473) {};
    \node (x153) at (0.6959, 0.9474) {};
    \node (x154) at (0.6874, 0.9464) {};
    \node (x155) at (0.6471, 0.9705) {};
    \node (x156) at (0.6394, 0.9465) {};
    \node (x157) at (0.6337, 0.9483) {};
    \node (x158) at (0.6004, 0.9864) {};
    \node (x159) at (0.5810, 0.9478) {};
    \node (x160) at (0.5476, 0.9676) {};
    \node (x161) at (0.5411, 0.9679) {};
    \node (x162) at (0.5062, 0.9601) {};
    \node (x163) at (0.5041, 0.9567) {};
    \node (x164) at (0.4877, 0.9513) {};
    \node (x165) at (0.4647, 0.9916) {};
    \node (x166) at (0.4352, 0.9711) {};
    \node (x167) at (0.4322, 0.9732) {};
    \node (x168) at (0.4226, 0.9749) {};
    \node (x169) at (0.4140, 0.9779) {};
    \node (x170) at (0.3886, 0.9679) {};
    \node (x171) at (0.3778, 0.9726) {};
    \node (x172) at (0.3584, 0.9471) {};
    \node (x173) at (0.3320, 0.9468) {};
    \node (x174) at (0.3294, 0.9410) {};
    \node (x175) at (0.3247, 0.9393) {};
    \node (x176) at (0.2985, 0.9419) {};
    \node (x177) at (0.2920, 0.8979) {};
    \node (x178) at (0.2919, 0.8978) {};

    \draw[fill=black!10] (0.5, 0.5) circle (0.5);
    \fill[opacity=1,white] (x1.center) \foreach \i in {2,...,178}{ -- (x\i.center) } -- cycle;
    \draw[mypath] (x1.center) \foreach \i in {2,...,178}{ -- (x\i.center) } -- (x1.center);
\end{tikzpicture}\hspace{12mm}\label{fig:approx-ball-2}
    }
\hfill
    \subfigure[Uniformly interpretable]{%
        \hspace{6mm}\begin{tikzpicture}[scale=3]
    \node (c) at (0.5,0.5) {};
    \filldraw[color=red!20] (c) circle (0.55);
    \node[regular polygon,draw,color=blue,fill=red!20,regular polygon sides=8,minimum size=8.5em] at (0.5,0.5) {};
    \draw[fill=white] (c) circle (0.5);
\end{tikzpicture}\hspace{4mm}\label{fig:approx-ball-3}
    }
    \caption{Approximating a disk with halfspaces: the approximation error is the grey-shaded area, while the pink area is the margin region. In (a), we show inapproximability with x-axis-aligned halfspaces. In (b), we show the disk is approximable (but not interpretable) with arbitrary halfspaces, via a Voronoi tessellation with one-sided error. In (c), we show the disk with margin is uniformly interpretable with halfspaces.}
    \label{fig:approx-ball}
\end{figure}

At first glance, our notions of interpretability may appear a little narrow. Suppose that, for every distribution $\Pd$, a concept $c$ is interpretable by $\scH$ at polynomial depth rate, but the degree can grow unbounded with $\Pd$. In other words, for every $d \in \NN_+$ there exists $\Pd_d$ such that $c$ is interpretable by $\scH$ at polynomial depth rate with degree $d$, but not at polynomial depth rate with any smaller degree $d' < d$. Then $c$ is not interpretable by $\scH$ according to our definition, but we could still say that $c$ is interpretable at \emph{polynomial} depth rate. More formally, we could consider the family $\scF$ of all polynomials, and require that for every $\Pd$ there is some $f \in \scF$ that bounds $\Depth{c}{\scH}{\cdot \mid \Pd}$. By varying $\scF$, we obtain a vast range of interpretability rates: logarithmic, sublinear, linear, polynomial, exponential, and so on. Surprisingly, our results show that this hierarchy collapses: an approximable concept $c$ is either not interpretable at all, or is uniformly interpretable at logarithmic rate.

\section{A Trichotomy for Interpretability}\label{sec:trichotomy}
This section states our main result: as soon as a concept is interpretable at \emph{some} rate, then it is uniformly interpretable at a constant rate for VC classes, and at a logarithmic rate in general.
\begin{theorem}[Interpretability trichotomy]
\label{thm:trichotomy}
    Let $X$ be any domain. For every concept $c$ and every VC hypothesis class $\scH$ over $X$ exactly one of the following cases holds:
    \begin{enumerate}[(1),itemsep=3pt,parsep=0pt]
    % No universal approximation:
    \item $c$ is not approximable by $\scH$.
    \item $c$ is approximable but not interpretable by $\scH$.
    \item $c$ is uniformly interpretable by $\scH$ at constant depth rate.
    \end{enumerate}
    If $\vc(X,\scH)=\infty$ then all claims above hold true, but with (3) replaced by:
    \begin{enumerate}[(1),itemsep=-2pt]
        \item[(3$'$)] \label{item:vc-infinite} $c$ is uniformly interpretable by $\scH$ at depth rate at most logarithmic.
    \end{enumerate}
    Moreover, all cases are nonempty.
    \end{theorem}
We emphasize again that Theorem~\ref{thm:trichotomy} is in stark contrast with the behavior of excess risk in terms of training set size observed in statistical learning, where, in the non-uniform (or universal) setting, both exponential and linear rates are possible. It should also be noted that we do not know if case (3$'$) collapses into case (3)---that is, if a constant depth rate holds also for non-VC classes---or if a non-constant rate is in general unavoidable. This is one of the questions the present work leaves open.

We further observe that, while point (3) of \Cref{thm:trichotomy} shows that $c$ is uniformly interpretable by $\scH$ at a constant depth rate, 
this does not necessarily imply the existence of a single $\scH$-based decision tree providing such a guarantee for all values of $\eps>0$.
For example, consider a domain $\scX = \NN$, a concept $c=\{0\}$, and a hypothesis class $\scH = \{ \{1,\dots,n\} : n \in \NN^+ \}$. Now let $\Pd$ be the distribution such that $\Pd(0)=0.5$ and $\Pd(x)=2^{-(x+1)}$ for all $x \in \NN^+$. For any $\eps > 0$ the depth-$1$ decision tree with splitting criterion $h=\{1,\ldots,\lceil\log_2(1/\eps)\rceil\}$ is an $\eps$-accurate approximation of $c$ under $\Pd$, but no $\scH$-based tree is an $\eps$-approximation of $c$ for all $\eps$ simultaneously.

Our proof of \Cref{thm:trichotomy} combines a variety of techniques from different contexts. The first step involves identifying a criterion which can be thought of as a form of ``weak interpretability'' (items~(a) and~(b) in the proof). The rest of the proof demonstrates that if a concept $c$ fails to satisfy this criterion, then it is not interpretable by $\scH$, and if it does, then it is uniformly interpretable by $\scH$. The former impossibility result entails establishing a lower bound on the interpretation rate for an arbitrarily small accuracy with respect to a \emph{fixed} and carefully tailored distribution.
This type of lower bounds are more intricate than distribution-free lower bounds (such as those outlined in the No-Free-Lunch Theorem in the PAC setting) and were studied, e.g., by \cite{AntosL98,BousquetHMST23}. In the complementary case, when $c$ satisfies the weak interpretability criterion with respect to $\scH$, we prove that $c$ is in fact uniformly interpretable by $\scH$ with logarithmic depth, and if $\scH$ has a finite VC dimension, then $c$ is interpretable with constant depth. The logarithmic construction and its analysis builds on ideas and techniques originating from boosting algorithms for decision trees~\citep{KearnsM99,TakimotoM03}. 
The derivation of constant depth approximation when $\scH$ is a VC class relies on a uniform convergence argument~\citep{vc1971} combined with the Minimax Theorem \citep{Neumann1928}. This derivation is also linked to boosting theory and resembles the boosting-based sample compression scheme by \cite{MoranY16}.

{\renewcommand{\proofname}{Proof of Theorem~\ref{thm:trichotomy}.}
\begin{proof}
We start by proving the cases (1)-(3).
Suppose (1) fails, so $\Depth{c}{\scH}{\eps \mid \Pd} < \infty$ for all $\Pd \in \PD(X)$ and all $\eps > 0$. This implies that, for any fixed $\gamma\in(0,\frac12)$, exactly one of the following two cases holds:
\begin{enumerate}[(a),itemsep=-2pt]
    \item for every $d \in \N$ there exists a distribution $\Pd_d$ such that $\Depth{c}{\scH}{\frac12-\gamma \mid \Pd_d} > d$;
    \item there exists $d \in \N$ such that $\Depth{c}{\scH}{\frac12-\gamma \mid \Pd} \le d$ for all distributions $\Pd$.
\end{enumerate}
Suppose (a) holds; we show this implies case (2) of the trichotomy. To this end, we prove that there is no function $r\colon (0,1] \to \N$ such that $c$ is interpretable by $\scH$ at depth rate $r$. Choose indeed any such $r$. For every $n \in \N^+$ let $d_n = r(2^{-n}(\frac12-\gamma))$, and consider the following distribution over $X$:
\begin{align}
    \Pd^* = \sum_{n \in \NN^+} 2^{-n} \cdot \Pd_{d_n} \enspace.
\end{align}
Since $\Pd_{d_n}$ appears in $\Pd^*$ with coefficient $2^{-n}$, this implies that, for $\eps_n = 2^{-n}(\frac12-\gamma)$, any $\eps_n$-accurate $\scH$-interpretation of $c$ under $\Pd^*$ is $(\frac12-\gamma)$-accurate under $\Pd_{d_n}$ and so has depth larger than $d_n=r(\eps_n)$. Hence,
\begin{equation}    
    \Depth{c}{\scH}{\eps_n \mid \Pd^*} \ge \Depth{c}{\scH}{\frac12-\gamma \mid \Pd_{d_n}} > d_n = r(\eps_n)
\end{equation}
holds for all $n\in\N^+$. We conclude that $c$ is not interpretable by $\scH$ at depth rate $r$, as desired.

Now suppose (b) holds; we show this implies case (3) of the trichotomy. 
Let $\scT$ be the set of all binary classifiers that are represented by $\scH$-based decision trees of depth at most~$d$, where $d \in \N$ satisfies $\Depth{c}{\scH}{\frac12-\gamma \mid P} \le d$ for all $P \in \PD(X)$. It is known that $\vc(X,\scH) < \infty$ implies $\vc(X,\scT) < \infty$~\citep{dudley78aggreg}. % for some function $f$ 

We will first prove the claim by taking as domain an arbitrary but finite subset $U \subseteq X$. Later on we will choose $U$ appropriately as a function of the distribution $P \in \PD(X)$, and this will prove the theorem's claim. Fix then any such $U$, and let $\PD(U)$ be the family of all distributions over $U$. By definition of $d$,
\begin{align}
    \sup_{\Pd \in \PD(U)} \inf_{T \in \scT} L_\Pd(T,c) \le \frac12 - \gamma \enspace. \label{eq:cx_hypothesis}
\end{align}
By von Neumann's minimax theorem, recalling that the value of the game does not change if the column player uses a pure strategy, we have that
\begin{align}
    \sup_{\Pd \in \PD(U)} \inf_{T \in \scT} L_\Pd(T,c) = \inf_{D \in \PD(\scT)} \sup_{x \in U} \E_{T \sim D} L_{\delta_x}(T,c) \enspace, \label{eq:cx_minimax}
\end{align}
where $\PD(\scT)$ is the set of all distributions over $\scT$, $\delta_x$ is the Dirac delta at $x \in U$, and $\E_{T \sim D} L_{\delta_x}(T,c)$ is thus the expected loss on $x$ of a tree $T$ drawn from $D$. Hence, there exists $D^*\in\PD(\scT)$ for which
\begin{align}
    \E_{T \sim D^*} L_{\delta_x}(T,c) \le \frac12 - \gamma \qquad \forall x \in U \enspace, \label{eq:cx_approx_1}
\end{align}
and therefore, since $c(x),T(x) \in \bool$ for all $x$ and $T$,
\begin{align}
    \bigl|c(x) - \Pr_{T \sim D^*}(T(x)=1)\bigr| = \Pr_{T \sim D^*}(T(x) \neq c(x)) \le \frac12 - \gamma \qquad \forall x \in U \enspace. \label{eq:cx_approx}
\end{align}
Let $(\scT,U)$ be the dual set system of $(U,\scT)$. 
Note that $\vc(\scT,U) \le \vc(\scT,X) < 2^{\vc(X,\scT)+1} < \infty$, where the second inequality shows a known relation \citep{assouad1983} between the primal VC dimension $\vc(X,\scT)$ of $(X,\scT)$ and its dual VC dimension $\vc(\scT,X)$. By the classic uniform convergence result of \citet{vc1971}, there exists a multiset $R\subseteq \scT$ with $|R| \le r \coloneqq r(\vc(X,\scT),\gamma,d)$ such that, for every $x \in U$,
\begin{align}
    \left| \frac{|\{T \in R : T(x)=1\}|}{|R|} - \Pr_{T \sim D^*}(T(x)=1) \right| < \frac{\gamma}{2} \enspace. \label{eq:approx_distr}
\end{align}
Together with~\eqref{eq:cx_approx} and~\eqref{eq:approx_distr} this yields
\begin{align}
    \left| \frac{|\{T \in R : T(x)=1\}|}{|R|} - c(x) \right| < \frac12 - 
\frac{\gamma}{2} \label{eq:TR}
\end{align}
by the triangle inequality.\footnote{Note that, if no $D^*$ achieves the infimum of the r.h.s.\ of \cref{eq:cx_minimax}, the same result holds with, say,
$(1-\gamma)/2$ as the r.h.s.\ of \cref{eq:cx_approx_1} because it suffices to show that the l.h.s.\ of \cref{eq:TR} is strictly less than $1/2$ for our purposes. 
}
We now build a $\scH$-based decision tree $T^*_U$ that computes the majority vote over all $T\in R$. This tree can be constructed as follows. Let $T_1,\ldots,T_{|R|}$ be the trees in $R$. Replace each leaf of $T_1$ with a copy of $T_2$; in the resulting tree replace every leaf with a copy of $T_3$, and so on until obtaining $T^*_U$. For each leaf $z\in\scL(T^*_U)$ of $T^*_U$, define its label $\ell_z$ as the majority vote given by leaves of (the copies of) $T_1,\ldots,T_{|R|}$ that are encountered on the path from the root of $T^*_U$ to $z$. Note that $T^*_U$ has depth bounded by $rd$ and, by~\eqref{eq:TR}, computes $c(x)$ for all $x \in U$. Thus, $L_U(T^*_U,c)=0$ where $L_U$ is the expected loss over the uniform distribution over $U$.

We now choose the set $U$ appropriately. Let $\scT^*$ be the family of all $\scH$-based decision trees whose depth is at most $rd$. Because, once again, $\vc(X,\scT^*) < \infty$, by uniform convergence there is a finite multiset $U \subseteq X$ such that, for all $T \in \scT^*$,
$
    \left |L_\Pd(T,c) - L_{U}(T,c)\right| \le \epsilon
$.
Since $T^*_U \in \scT^*$ and $L_U(T^*_U,c)=0$, it follows that $L_{\Pd}(T^*_U,c) \le \epsilon$. This completes the proof of case~(3).
Case (3$'$) follows from \Cref{thm:general_logeps_rate} below, assuming (b) holds.

It remains to prove that all cases are nonempty.
For (1) let $X=\{a,b\}$, $\scH=\{X\}$, $c=\{a\}$, and note that under the uniform distribution no $\scH$-interpretation of $c$ is $\epsilon$-accurate for $\epsilon < \frac12$.
For (3) consider any $X,\scH$ with $\scH \ne \emptyset$ and choose any $c \in \scH$; this holds for (3$'$) too if $\scH$ is not a VC class.
For (2) we show $c,\scH$ that satisfy case (a) above.
Let $X=\NN$, $c=\NN^+$, and $\scH=\{\{i\}: i \in \NN^+\}$. For every $n \in \NN^+$ consider the distribution $\Pd_n$ with support $\{0,\ldots,n\}$ such that $\Pd_n(0)=\frac{1}{2}$ and that $\Pd_n(i)=\frac{1}{2n}$ for every $i \in \{1,\ldots,n\}$. To conclude note that $\Depth{c}{\scH}{\frac12-\gamma \mid \Pd_n}$ is unbounded as a function of $n$ for any constant $\gamma\in(0,\frac12)$.
\end{proof}
}

The proof of case (3$'$) of \Cref{thm:trichotomy} uses the following result. Its proof can be found in Appendix~\ref{section:boosting}, and is an adaptation of the results by \citet{KearnsM99} and \citet{TakimotoM03} on boosting decision trees. The main difference is that, via an adequate modification of the \texttt{TopDown} algorithm \citep{KearnsM99}, we bound the depth rather than the size of the boosted decision tree.
\begin{restatable}{theorem}{generalLogepsRate}\label{thm:general_logeps_rate}
Let $X$ be any domain. For any concept $c$ and any hypothesis class $\scH$ over $X$, if there exist $\gamma \in (0,\frac12)$ and $d \in \NN$ such that $\Depth{c}{\scH}{\frac{1}{2}-\gamma \mid \Pd} \le d$ for all $\Pd \in \PD(X)$, then $\Depth{c}{\scH}{\eps \mid \Pd} \le \frac{d}{2\gamma^{2}}\log\frac{1}{2\eps}$ for all $\Pd \in \PD(X)$ and all $\eps > 0$. 
\end{restatable}

\section{Algebraic Characterizations}
\label{sec:algebraic}
In this section we show that the notions of approximability and interpretability admit set-theoretical and measure-theoretical characterizations based on properties of $\scH$ and the algebras it generates.

To begin with, we need a notion of closure of $\scH$. Loosely speaking, we want to include all concepts that, under every distribution, can be approximated arbitrarily well by single elements of $\scH$. In other words, these are the concepts that are approximable by $\scH$ using decision trees of depth $1$.
\begin{definition}
The \emph{closure} of $\scH \subseteq 2^X$ is
\begin{equation}
    \clos(\scH)
=
    \Bigl\{h\subseteq X \,\big\vert\, \forall \Pd\in\PD(X),\,\exists h_1,h_2,\ldots\in\scH \,\text{s.t.}\, \lim\limits_{n\rightarrow\infty} \Pd(h\triangle h_n)=0 \Bigr\} \enspace.
\end{equation}
\end{definition}
Observe that $\clos(\scH) \supseteq \scH$ by definition. To illustrate the closure let us discuss the hypothesis class $\scH$ of rational halfspaces in $\RR^2$, i.e., sets of the form $\bigl\{\{(x,y) \mid ax + by + d\geq 0\}: a,b,d\in\mathbb{Q}\bigr\}$.
Every concept $c\colon \mathbb{R}^2\to \bool$ is approximable by $\scH$, as before, relying on the $1$-NN algorithm.
 Halfspaces with real coefficients such as $\{(x,y)\mid x+y\geq \sqrt{2}\}$ are not in $\scH$ but are interpretable by $\scH$ with depth~$1$. In general, the closure is related to the concept of universally measurable sets.

We start with the following lemma, which is derived from well-known results in measure theory (see the proof in Appendix~\ref{section:further_proofs}).
\begin{restatable}{lem}{closOfAlgIsCloseOfSigma}
\label{lem:clos_of_alg_is_clos_of_sigma}
    Let $X$ be any domain and $\scH\subseteq 2^X$.
    Then, $\clos(\sigma(\scH)) = \clos(\Alg(\scH))$.
\end{restatable}

We now state the algebraic characterization of the concepts that are approximable by a given hypothesis class $\scH$ on some domain $X$.
\begin{theorem}[Algebraic characterization of approximability]
\label{thm:algebraic_univ_approx}
    Let $X$ be any domain and $\scH$ any hypothesis class over $X$. A concept $c\subseteq X$ is approximable by $\scH$ if and only if $c \in \clos(\sigma(\scH))$.
\end{theorem}
\begin{proof}
    Suppose $c$ is universally approximable by $\scH$.
    Let $\Pd \in \PD(X)$ be any distribution. Then, for every $\eps > 0$ there exists an $\eps$-accurate $\scH$-approximation $T \in \Alg(\scH)$ of $c$ under $\Pd$. Then $\Pd(T \triangle c) = L_\Pd(T, c) \le \eps$. Consider now the sequence $T_1,T_2,\ldots \in \Alg(\scH)$ such that, for each ${n\in\N_+}$, $T_n$ is an $\eps_n$-accurate $\scH$-approximation of $c$ under $\Pd$ with the choice $\eps_n=2^{-n}$. The sequence $(T_n)_{n\in\N_+}$ is such that $\lim_{n\to\infty} \Pd(T_n \triangle c) \le \lim_{n\to\infty} 2^{-n} = 0$, and thus $c \in \clos(\Alg(\scH))=\clos(\sigma(\scH))$, where the latter equality follows by Lemma~\ref{lem:clos_of_alg_is_clos_of_sigma}.

    Now suppose $c \in \clos(\sigma(\scH))=\clos(\Alg(\scH))$.
    Fix a distribution $\Pd \in \PD(X)$ and $\eps>0$. By definition of closure, and because $\Alg(\scH) \equiv \AllTrees$, there exists a sequence $T_1,T_2,\ldots \in \Alg(\scH)$ of trees such that $\lim_{n \to \infty} \Pd(T_n \triangle c) = 0$, and thus there exists some $i\in\NN_+$ such that $\Pd(T_i \triangle c) \le \eps$. This implies that $T_i$ is an $\eps$-accurate $\scH$-approximation of $c$ under $\Pd$ with finite depth. As this holds for every $\Pd$ and every $\eps > 0$, it follows that $c$ is universally approximable by $\scH$.
\end{proof}
Furthermore, we manage to prove an algebraic characterization for the concepts that are uniformly interpretable, given a VC class $\scH$ on some domain $X$.
\begin{theorem}[Characterization of uniform interpretability for VC classes] \label{thm:algebraic_uniform_vc}
Let $X$ be any domain and let $\scH$ be a VC hypothesis class over $X$. A concept $c$ is uniformly interpretable (at a constant depth) if and only if $c\in\bigcup_{d=1}^\infty \clos(\Alg_d(\scH))$.
\end{theorem}
\begin{proof}
Since $\vc(X,\scH)<\infty$, item~(3) of Theorem~\ref{thm:trichotomy} implies that there exists $d \in \N$ such that, for all $\Pd \in \PD(X)$ and all $\eps > 0$, $\Depth{c}{\scH}{\eps \mid \Pd}\leq d$. Using an argument similar to the one used in the proof of Theorem~\ref{thm:algebraic_univ_approx}, we then conclude that $c \in \clos(\Alg_d(\scH))$. Hence, the set of concepts that are uniformly interpretable by $\scH$ is precisely $\bigcup_{d=1}^\infty \clos(\Alg_d(\scH))$.
\end{proof}

If the domain $X$ is countable, then closure reduces to pointwise convergence and our algebraic characterization becomes simpler. This is formalized in the next theorem, whose proof relies on some technical lemmas and can be found in \Cref{section:further_proofs}. Namely, we show that $\clos(\sigma(\scH))=\sigma(\scH)$ and $\bigcup_{d=1}^\infty \clos(\Alg_d(\scH)) = \Alg(\clos(\scH))$.
\begin{restatable}{theorem}{algebraicCountable}\!\emph{\textbf{(Characterization for VC classes and countable domains)}}\;\label{thm:algebraic_countable}
Let $X$ be any countable domain, let $c$ be any concept, and let $\scH$ be a VC hypothesis class over $X$. Then:
\begin{enumerate}[itemsep=2pt,parsep=0pt,topsep=0pt]
    \item $c$ is approximable by $\scH$ if and only if $c \in \sigma(\scH)$.
    \item $c$ is approximable but not interpretable by $\scH$ if and only if $c\in \sigma(\scH) \setminus \Alg(\clos(\scH))$.
    \item $c$ is uniformly interpretable by $\scH$ if and only if $c\in \Alg(\clos(\scH))$.
\end{enumerate}
\end{restatable}
\section{General Representations}
\label{sec:gen-rep}
Although shallow decision trees are the blueprint of interpretable models, our theory naturally extends to ways of measuring the complexity of elements in $\Alg(\scH)$ different from the tree depth. Next, we define a set of minimal conditions (satisfied, e.g., by tree depth) that a function must satisfy to be used as a complexity measure for $\Alg(\scH)$.
\begin{definition}\label{def:gcm}
    Let $X$ be any domain and $\scH$ a hypothesis class over $X$. A function $\Gamma\colon\Alg(\scH)\to\NN$ is a \emph{graded complexity measure} if:
    \begin{enumerate}[itemsep=2pt,parsep=0pt]
        \item $\Gamma(f)= 0$ for all $f\in\scH$,
        \item $\Gamma(f_1\cup f_2)\leq 1+\Gamma(f_1)+\Gamma(f_2)$ for all $f_1,f_2\in\Alg(\scH)$, 
        \item $\Gamma(f_1\cap f_2)\leq 1+\Gamma(f_1)+\Gamma(f_2)$ for all $f_1,f_2\in\Alg(\scH)$, and
        \item $\Gamma(X\setminus f)\leq 1 + \Gamma(f)$ for all $f\in\Alg(\scH)$.
    \end{enumerate}
    The minimal complexity of an $\epsilon$-accurate $\scH$-interpretation of $c$ under $\Pd$ is
    \begin{equation} 
        \GCM{c}{\scH}{\eps \mid \Pd} =
        \inf_{T\in\Alg(\scH) :\, L_\Pd(T,c)\leq \eps} \Gamma(T)
        \enspace.
    \end{equation}
\end{definition}
The definitions of approximability, interpretability, and uniform interpretability are readily generalized to an arbitrary graded complexity measure, by simply replacing $\depth(\cdot)$ with $\GCM{}{}{\cdot}$. We can then prove the following extension of \Cref{thm:trichotomy}.
\begin{theorem}[Interpretability trichotomy for general representations]
\label{thm:trichotomy-gcm}
    Let $X$ be any domain and let $\Gamma$ be any graded complexity measure. Then, for every concept $c$ and every VC hypothesis class $\scH$ over $X$ exactly one of the following cases holds:
    \begin{enumerate}[(1),itemsep=2pt,parsep=0pt,topsep=0pt]
    \item $c$ is not approximable by $\scH$.
    \item $c$ is approximable by $\scH$ but not interpretable by $\scH$.
    \item $c$ is uniformly interpretable by $\scH$ at constant $\Gamma$-complexity rate.
    \end{enumerate}
    If $\vc(X,\scH)=\infty$ then all claims above hold true, but with (3) replaced by:
    \begin{enumerate}[(1),itemsep=-2pt,topsep=0pt]
        \item[(3$'$)] $c$ is uniformly interpretable by $\scH$ at a $\Gamma$-complexity rate $\scO\bigl(\frac{1}{\eps^d}\bigr)$ for some $d\in\NN$ .
    \end{enumerate}
\end{theorem}

Unlike \Cref{thm:trichotomy}, cases (2) and (3) might collapse for certain choices of $\Gamma$ even when $\scH$ is not a VC class. Indeed, according to our definition, $\Gamma$ is not forced to grow at any specific rate, and thus $\Gamma(f)$ might be bounded by some constant uniformly over $\Alg(\scH)$. In an extreme case one might in fact set $\Gamma\equiv 0$, although clearly this would not yield any interesting result.
{\renewcommand{\proofname}{Proof of Theorem~\ref{thm:trichotomy-gcm}.}
\begin{proof}
The proof is similar to the proof of Theorem~\ref{thm:trichotomy}. Suppose (1) fails, so $\GCM{c}{\scH}{\eps \mid \Pd} < \infty$ for all $\eps > 0$ and all distributions $\Pd$. This implies that, for any fixed $\gamma\in(0,\frac12)$, exactly one of the following two cases holds:
\begin{enumerate}[(a),itemsep=2pt,parsep=0pt,topsep=0pt]
    \item for every $k \in \NN$ there exists a distribution $\Pd_k$ such that $\GCM{c}{\scH}{\frac12-\gamma \mid \Pd_k} > k$;
    \item there exists $k \in \NN$ such that $\GCM{c}{\scH}{\frac12-\gamma \mid \Pd} \le k$ for all distributions $\Pd$.
\end{enumerate}
Suppose (a) holds; we show this implies case (2) of the trichotomy. Choose any function $r\colon (0,1]\to\NN$. For every $n \in \N^+$ let $d_n = r(2^{-n}(\frac12-\gamma))$, and consider the following distribution over $X$:
\begin{align}
    \Pd^* = \sum_{n \in \NN^+} 2^{-n} \cdot \Pd_{d_n} \enspace.
\end{align}
Since $\Pd_{d_n}$ appears in $\Pd^*$ with coefficient $2^{-n}$, this implies that, for $\eps_n = 2^{-n}(\frac12-\gamma)$, any $\eps_n$-accurate interpretation of $c$ under $\Pd^*$ is $(\frac12-\gamma)$-accurate under $\Pd_{d_n}$, and thus
\begin{align}
    \GCM{c}{\scH}{\eps_n \mid \Pd^*} \ge \GCM{c}{\scH}{\frac12-\gamma \mid \Pd_{d_n}} > d_n = r(\eps_n) \enspace.
\end{align}
Hence, $\GCM{c}{\scH}{\eps_n \mid \Pd^*} > r(\eps_n)$ for all $n\in\N^+$.

Suppose now (b) holds; we show this implies case (3) of the trichotomy.
Define the family $\scA_k = \bigl\{A \in \Alg(\scH) \,:\, \Gamma(A) \le k\bigr\}$.
Fix any $\Pd \in \PD(X)$ and $\eps > 0$.
Following the same argument as in the proof of case~(3) in Theorem~\ref{thm:trichotomy}, there exists an $\scA_k$-based decision tree $T$ such that $L_{\Pd}(T,c) \le \epsilon$ and $\depth(T) \le d$ for some $d \in \NN$ independent of $\Pd$ and $\eps$.
Now we rewrite $T$ as an element of $\Alg(\scH)$.
Let $A_v \in \scA_k$ be the decision stump $T$ used at $v \in \scV(T)$ and, denoting by $\scL(T)$ the set of leaves of $T$, let $\ell_z\in\bool$ be the label of the leaf $z\in\scL(T)$ in $T$.
For every $v \in \scV(T)$, define
\begin{align} \label{eq:gcm_tree_to_alg}
    A_v^{T} = \left\{
    \begin{array}{ll}
        X & v \in \scL(T), \; \ell_v = 1 \\
        \emptyset & v \in \scL(T), \; \ell_v = 0 \\
        (A_v \cap A_u^T) \cup (\overline A_v \cap A_w^T)  &  v \notin \scL(T)
    \end{array}
    \right.
\end{align}
where $u$ and $w$ are, respectively, the left and right child of $v$ when $v\notin \scL(T)$. Let $A=A_r^T$ where $r$ is the root of~$T$. Observe that $A$ is equivalent to $T$, and that $A \in \Alg(\scH)$. Moreover, $\Gamma(A_v^T) \le 4 + 2\Gamma(A_v) + \Gamma(A_u^T) + \Gamma(A_w^T)$ by the properties of $\Gamma$ (see Definition~\ref{def:gcm}). Therefore,
\begin{align}
    \Gamma(A) = \scO\Biggl(\sum_{v \in \scV(T)} (\Gamma(A_v)+1)\Biggr) = (k+1)\times\scO(|\scV(T)|) = \scO(|\scV(T)|) \enspace,
\end{align}
where we used the fact that $\Gamma(A_v) \le k$ because $A_v \in \scA_k$.
To conclude the proof, note that the above bound on $\depth(T)$ implies $\scO(|\scV(T)|) = \scO(2^{\depth(T)}) = \scO(2^d)$, where both $d$ and the constants in the $\scO(\cdot)$ notation depend neither on $\Pd$ nor on $\eps$.

As for case (3$'$), assume again (b) holds. Then, Theorem~\ref{thm:general_logeps_rate} applied to the class $\scA_k$ implies the existence of an $\scA_k$-based decision tree $T$ such that $L_{\Pd}(T,c) \le \epsilon$ and $\depth(T) \le d\log\frac{1}{2\epsilon}$ for all $P$ and $\eps > 0$, where $d=\frac{1}{2\gamma^2}$. Constructing again $A\in\Alg(\scH)$ equivalent to $T$ as above and using the bound on $\depth(T)$, we have $\Gamma(A) = \scO(|\scV(T)|) = \scO(2^{\depth(T)}) = \scO\bigl(\frac{1}{\eps^d}\bigr)$ where both $d$ and the constants in the $\scO(\cdot)$ notation are independent of $\Pd$ and $\eps$.
\end{proof}
}

We remark that the $\scO\bigl(\frac{1}{\eps^d}\bigr)$ bound on the complexity rate for case (3$'$) is due to the generality of $\Gamma$; in Appendix~\ref{apx:gcm_remark}, we show a more specific condition on $\Gamma$ that recovers the $\scO(\log(1/\eps))$ rate.

\acks{MB, NCB, and EE acknowledge the financial support from the FAIR (Future Artificial Intelligence Research) project, funded by the NextGenerationEU program within the PNRR-PE-AI scheme and the the EU Horizon CL4-2022-HUMAN-02 research and innovation action under grant agreement 101120237, project ELIAS (European Lighthouse of AI for Sustainability).

SM is supported by a Robert J.\ Shillman Fellowship, by ISF grant 1225/20, by BSF grant 2018385, by an Azrieli Faculty Fellowship, by Israel PBC-VATAT, and by the Technion Center for Machine Learning and Intelligent Systems (MLIS), and by the European Union (ERC, GENERALIZATION, 101039692). Views and opinions expressed are however those of the author(s) only and do not necessarily reflect those of the European Union or the European Research Council Executive Agency. Neither the European Union nor the granting authority can be held responsible for them.

YM has received funding from the European Research Council (ERC) under the European Union’s Horizon 2020 research and innovation program (grant agreement No. 882396), by the Israel Science Foundation,  the Yandex Initiative for Machine Learning at Tel Aviv University and a grant from the Tel Aviv University Center for AI and Data Science (TAD).

MT acknowledges support from a DOC fellowship of the Austrian academy of sciences (ÖAW).
}

\bibliography{refs.bib}

\newpage

\appendix
%%%%%appendix_boosting
\section{Boosting Decision Trees with Bounded Depth}\label{section:boosting}
\generalLogepsRate*

We use a surrogate loss $G(q)=\sqrt{q(1-q)}$, where $0 \le q \le 1$. Since $\min\{q,1-q\} \le G(q)$, the surrogate loss bounds from above the classification error of the majority vote.
For a distribution $\Pd \in \PD(X)$, let $G_\Pd(c)=G\bigl(\Pd(c=1)\bigr)$. Let the conditional surrogate loss of $f\colon X \to \bool$ be
\begin{equation}
    G_{\Pd}(c \mid f) = \Pd(f=0) G\bigl(\Pd(c=1 \mid f=0)\bigr) + \Pd(f=1) G\bigl(\Pd(c=1 \mid f=1)\bigr) \enspace.
\end{equation}
Finally, given a decision tree $T$ with leaves $\scL(T)$, define the conditional surrogate loss of $T$ as
\begin{equation}
    H_{\Pd}(c \mid T) = \sum_{z\in\scL(T)} \Pd(z) G(p_{c\mid z}) \enspace,
\end{equation}
where $\Pd(z)$ is the probability that $x\sim\Pd$ is mapped to leaf $z$ in the tree $T$ and $p_{c\mid z} = \Pd(c=1 \mid z)$.
Our goal is to construct an $\scH$-based decision tree $T$ such that $H_{\Pd}(c \mid T) \leq \epsilon$, implying that $L_{\Pd}(T,c) \le \epsilon$ because $G\bigl(p_{c\mid z}\bigr)$ bounds from above the probability that $T(x) \neq c(x)$ conditioned on $x$ being mapped to $z$ in $T$.

Our variant of \texttt{TopDown}, called \texttt{TopDownLBL} (TopDown Level-By-Level), starts from a single-leaf tree $T$ with a majority-vote label and works in phases. In each phase, we replace each leaf $z\in\scL(T)$ of the current tree $T$ with a suitably chosen $\scH$-based $d$-depth tree $T_{z}$ using the same criterion as \texttt{TopDown}. The main difference is that the weak learners adopted by \texttt{TopDownLBL} consist of $\scH$-based trees of depth bounded by $d$, which generalize from the individual decision stumps of $\scH$ as in \texttt{TopDown} (corresponding to the case $d=1$). Hence, at the end of each phase, the depth of $T$ increases by at most~$d$. The algorithm stops if and when $H_{\Pd}(c \mid T) \le \epsilon$.

We use the two following lemmas.
\begin{lemma}[\protect{\citet[Lemma~A.1]{TakimotoM03}}]
\label{lem:a.1}
    Let $\Pd$ be a balanced distribution, i.e., $\Pd(c=1) = \Pd(c=0)=\frac12$.
    Let $f\colon X \to \bool$ be such that $L_{\Pd}(f,c) \leq \frac12-\gamma$ for some $\gamma \in (0,\frac12)$. Then, $G_{\Pd}(c \mid f)\leq (1-2\gamma^2)G_{\Pd}(c)$.
\end{lemma}
\begin{lemma}[\protect{\citet[Proposition~5]{TakimotoM03}}]
\label{lem:prop-5}
    Let $\Pd$ be a distribution and $\Pd'$ its balanced version.
    If $G_{\Pd'}(c \mid h)\leq (1-\beta)G_{\Pd'}(c)$ for some $\beta>0$ then $G_{\Pd}(c \mid h)\leq (1-\beta)G_{\Pd}(c)$.
\end{lemma}
{\renewcommand{\proofname}{Proof of Theorem~\ref{thm:general_logeps_rate}.}
\begin{proof}

Our algorithm \texttt{TopDownLBL} can be equivalently viewed as building a $\scH'$-based tree $T'$, where $\scH'$ is the class of $\scH$-based $d$-depth trees. Any $\scH'$-based tree $T'$ can be transformed into a $\scH$-based tree $T$ in a top-down fashion simply by listing the nodes at each level of $T'$ starting from the root, and iteratively replacing every decision stump $h'\in\scH'$ with the corresponding $\scH$-based tree $T_{h'}$. Then, each leaf $z \in \scL(T_{h'})$ of $T_{h'}$ is replaced by copies of the left or right subtree of the decision stump $h'$ in $T'$ based on the values ($0$ or $1$) of the label $\ell_z$ of $z$. Clearly, the depth of $T$ is at most $d$ times the depth of $T'$.

We now bound the drop in $H_{\Pd}(c \mid T')$ when a leaf $z$ in the $\scH'$-based tree $T'$ is replaced by a decision stump in $\scH'$. Let $\Pd$ the distribution over $X$ conditioned on $x$ being mapped to $z$ and let $\Pd'$ its ``balanced'' version satisfying $\Pd'(c=1) = \Pd'(c=0) = \frac12$. Because of our weak learning assumption, we know there exists $h'_{z}\in\scH'$ with error at most $1/2-\gamma$ on $\Pd'$. By Lemma~\ref{lem:a.1}, $G_{\Pd'}(c \mid h'_{z})\leq (1-2\gamma^2)G_{\Pd'}(p'_{c\mid z})$, where $p'_{c\mid z} = \frac{1}{2}$ because of the balanced property of $P'$. Hence, by Lemma~\ref{lem:prop-5},
\begin{equation}
\label{eq:drop}
    G_{\Pd}(c \mid h'_{z})
\le
    (1-2\gamma^2)G_{\Pd}(p_{c\mid z}) \enspace.
\end{equation}
Let $T_{z}'$ be the tree $T'$ in which we replaced a leaf $z\in\scL(T')$ with the decision stump $h'_{z}\in\scH'$. Using \Cref{eq:drop},
\begin{equation}
    H_{\Pd}(c \mid T') - H_{\Pd}(c \mid T_{z}')
=
    \bigl( G_{\Pd}(p_{c\mid z}) - G_{\Pd}(c \mid h'_{z}) \bigr)\Pd(z)
\ge
    2\gamma^2G_{\Pd}(p_{c\mid z})\Pd(z) \enspace.
\end{equation}
Now let $T'_i$ be the tree after the algorithm has run for $i$ phases. Using the above inequality for each $z\in\scL(T'_i)$, we obtain
\begin{equation}
    H_{\Pd}(c \mid T'_i) - H_{\Pd}(c \mid T'_{i+1})
\ge
    \sum_{z\in\scL(T'_i)} 2\gamma^2 G_{\Pd}(p_{c\mid z}) \Pd(z)
=
    2\gamma^2 H_{\Pd}(c \mid T'_i) \enspace.
\end{equation}
Hence, after $m$ phases,
\begin{equation}
    L_\Pd(T'_m,c)
\le
    H_{\Pd}(c \mid T'_m) \le \bigl(1-2\gamma^2\bigr)^m H_{\Pd}(c \mid T'_0)
\le
    \frac{1}{2} e^{-2m\gamma^2} \enspace,
\end{equation}
where $T'_0$ is the initial tree consisting of a single leaf $z$ and, in the last inequality, we used the fact that $H_{\Pd}(c \mid T'_0) = G_{\Pd}(p_{c\mid z}) \le \frac{1}{2}$ and the inequality $1-x\le e^{-x}$. The proof is concluded by noting that $\frac{1}{2} e^{-2m\gamma^2} \le \eps$ for $m \ge \frac{1}{2\gamma^2}\log\frac{1}{2\varepsilon}$.
\end{proof}
}
We remark that we recover the standard setting of boosting decision trees when $d=1$. In this special case, our result matches the depth lower bound mentioned by \citet{KearnsM99}, while guaranteeing a $\scO\bigl(2^{\Depth{c}{\scH}{\eps \mid P}}\bigr) = \scO\bigl((1/\eps)^{1/(2\gamma^2)}\bigr)$ tree-size upper bound that is analogous to the ones by \citet{KearnsM99} and \citet{TakimotoM03}.

%%%%%%appendix_proofs
\section{Further Proofs for the Algebraic Characterization}\label{section:further_proofs}

\subsection{Algebraic characterization}

\closOfAlgIsCloseOfSigma*
{\renewcommand{\proofname}{Proof of Lemma~\ref{lem:clos_of_alg_is_clos_of_sigma}.}
\begin{proof}
    We begin with the proof of the first identity in the statement.
    The inclusion $\clos(\Alg(\scH)) \subseteq \clos(\sigma(\scH))$ immediately follows by definition of closure.
    We now show that the converse is also true.
    Let $T \in \clos(\sigma(\scH))$. Fix a distribution $\Pd \in \PD(X)$ and $\eps>0$. By definition of closure, there exists a sequence $A_1,A_2,\ldots \in \sigma(\scH)$ such that $\lim_{i \to \infty} \Pd(A_i \triangle T) = 0$. Consequently, for every $\eps > 0$ there exists some $i\in\N^+$ such that $\Pd(A_i \triangle T) \le \eps$. Thus, we can assume without loss of generality that the sequence $(A_i)_{i\in\N^+}$ satisfies $\Pd(A_i \triangle T) \le \eps_i$ for the choice $\eps_i = 2^{-i}$, for each $i\in\N^+$ (as we can select such a subsequence).
    Denote the restriction of $\Pd$ to $\sigma(\scH)$ as $\restr{\Pd}{\sigma(\scH)}$, that is $\restr{\Pd}{\sigma(\scH)}\colon\sigma(\scH)\to\RR_{\ge0}$ and $\restr{\Pd}{\sigma(\scH)}(A)=\Pd(A)$ for all $A\in\sigma(\scH)$. It is well known that, for each $i$, we can select an element $B_i\in\Alg(\scH)$ with $\restr{\Pd}{\sigma(\scH)}(B_i\triangle A_i)\leq \eps_i$ (see, e.g., \citet[Theorem~D, Section~13]{halmos2013measure}); hence, $\Pd(B_i\triangle A_i)\leq \eps_i$.
    By the triangle inequality $\Pd(T \triangle B_i) \le 2\eps_i = 2^{-i+1}$ for any $i$, which also implies that $\lim_{i\to\infty} \Pd(T \triangle B_i) = 0$ for the sequence $(B_j)_{j\in\N^+}$ in $\Alg(\scH)$. Therefore, $T\in\clos(\Alg(\scH))$.
\end{proof}
}

\subsection{Algebraic characterization for countable domains}
\algebraicCountable*
{\renewcommand{\proofname}{Proof of Theorem~\ref{thm:algebraic_countable}}
\begin{proof}
    Item 1 follows from Lemma~\ref{lem:countable_clos_sigma} Theorems~\ref{thm:algebraic_univ_approx} and~\ref{thm:algebraic_uniform_vc}, and items 2 and 3 from Lemma~\ref{lem:dfinterp_alg_clos}.
\end{proof}
}

\begin{lemma}\label{lem:countable_clos_sigma}
    Let $X$ be any countable domain and $\scH$ be any hypothesis class over $X$.
    Then, $\clos(\Alg(\scH)) = \sigma(\scH)$.
\end{lemma}
\begin{proof}
    Clearly $\sigma(\scH) \subseteq \clos(\sigma(\scH)) = \clos(\Alg(\scH))$ by Lemma~\ref{lem:clos_of_alg_is_clos_of_sigma}. Now we prove the converse. Let $A \in \clos(\sigma(\scH))$ and let $\Pd \in \PD(X)$ such that $\Pd(x)>0$ for all $x\in X$; note that $\Pd$ exists as $X$ is countable and it also means that $\supp(\Pd)=X$. By definition of $\clos(\sigma(\scH))$, there exists a sequence $(A_i)_{i \in \N^+}$ in $\sigma(\scH)$ such that
    \begin{align}
        \lim_{i \to \infty} \Pd(A \triangle A_i) = 0 \enspace.
    \end{align}
    By selecting an appropriate subsequence, we can assume $\Pd(A \triangle A_i) \le 2^{-i}$ for all $i \in \N^+$ without loss of generality. Define
    \begin{align}
        B_i = \bigcap_{j \ge i} A_j \qquad \forall i \in \N^+
    \end{align}
    and observe that $B_i \in \sigma(\scH)$ for each $i\in\N^+$. Note that
    \begin{align}
        \Pd(A \triangle B_i) = \Pd\biggl(A \triangle \bigcap_{j \ge i} A_j\biggr) \le \sum_{j \ge i} \Pd(A \triangle A_j) \le 2^{-i+1}\enspace.
    \end{align}
    Note also that $B_i \subseteq A$ for all $i \in \N^+$. Suppose indeed this was not the case, then $B_i \setminus A \neq \emptyset$. Hence, by definition of $B_i$, there exists some $x \in A_j \setminus A$ for all $j \ge i$. Since $\Pd(x)>0$ by the choice of $\Pd$, we have the contradiction
    \begin{align}
        \lim_{i \to \infty} \Pd(A \triangle A_i) \ge \Pd(x) > 0\enspace.
    \end{align}
    Now consider the set
    \begin{align}
        B = \bigcup_{i \in \N^+} B_i = \lim_{i\to\infty} B_i\enspace.
    \end{align}
    Note that by construction $B \in \sigma(\scH)$.\footnote{In particular, $B = \liminf_{i\to\infty} A_i$.}
    Moreover, since $B_i \subseteq B_{i+1}$ and $B_i \subseteq A$ for all $i \in \N^+$, we have that the sequence $(A\triangle B_i)_{i\in\N^+}$ is downward monotone and thus
    \begin{align}
        \Pd(A \triangle B) = \Pd\biggl(\bigcap_{i\in\N^+} A\triangle B_i\biggr) = \lim_{i \to \infty} \Pd(A \triangle B_i) = 0\enspace.
    \end{align}
    Given that $\Pd$ has full support, this implies $A=B$.
\end{proof}

\begin{definition}\label{def:pointwise_convergence}
    Let $X$ be any set. A sequence $(h_i)_{i \in \N}$ in $2^X$ is \emph{pointwise convergent} to $h \in 2^X$ if 
    \begin{align}
        \forall x \in X \quad \exists i_x \in \N :\; \forall i \ge i_x \qquad x \in h_i \iff x \in h \enspace.
    \end{align}
\end{definition}

\begin{proposition}\label{pro:countable_compact}
    If $X$ is countable then every infinite sequence $(h_i)_{i \in \N}$ in $2^X$ contains an infinite subsequence that is pointwise convergent.
\end{proposition}
Let $\scH \subseteq 2^X$ and let $\clospw(\scH)$ be the family of all subsets of $X$ that are the pointwise limit of some sequence in $\scH$. Clearly $\scH \subseteq \clospw(\scH) \subseteq 2^X$.
\begin{lemma}\label{lem:clos=clospw}
    If $X$ is countable then $\clospw(\scH)=\clos(\scH)$.
\end{lemma}
\begin{proof}
    To see that $\clospw(\scH) \subseteq \clos(\scH)$, recall the definition of pointwise convergence, and note how it implies that if a sequence $(h_i)_{i \in \N}$ converges pointwise to $h$ then $\lim_{i \to \infty} \Pd(h_i \triangle h) = 0$ for every $\Pd \in \PD(X)$. To see that $\clospw(\scH) \supseteq \clos(\scH)$, choose any sequence $(h_i)_{i \in \N}$ that converges to some $h \in \clos(\scH)$ under an appropriate distribution $\Pd \in \PD(X)$ such that $\supp(\Pd)=X$ (which exists as $X$ is countable); observe that this implies the pointwise convergence of $(h_i)_{i \in \N}$ to $h$.
\end{proof}

\begin{lemma}\label{lem:dfinterp_alg_clos}
If $X$ is countable and $\scH$ is a VC class over $X$, then $\clos(\Alg_d(\scH)) \subseteq \Alg(\clos(\scH))$ for every $d \in \N$.
\end{lemma}
\begin{proof}
Let $d\in\N$ and $c \in \clos(\Alg_d(\scH))$. By Lemma~\ref{lem:clos=clospw}, $c \in \clospw(\Alg_d(\scH))$, so there exists an infinite sequence of trees $(T_i)_{i \in \N}$ in $\Alg_d(\scH)$ that converge pointwise to $c$. Without loss of generality, we may assume that every $T_i$ is a complete tree of depth $d$.\footnote{One can always complete $T_i$ using internal nodes that hold, \eg the decision rule of the root.} Now consider the sequence $(h^1_i)_{i \in \N}$ of decision rules used by the first node (say, the root) of those trees. By Proposition~\ref{pro:countable_compact} there is an infinite subsequence $(h^1_{i_j})_{j \in \N}$ that is pointwise convergent to some $h^1\in\scH$. Now consider the infinite sequence of trees $(T_{i_j})_{j \in \N}$, and repeat the argument for the second node (say, a child of the root corresponding to a specific output of the decision stump at the root). By repeating the argument $2^d-1$ times (one for every internal node of the trees) we obtain an infinite sequence $(T^*_i)_{i \in \N}$ of trees in $\Alg_d(\scH)$ that converge pointwise to $c$ and such that at every node $v$ the decision rules converge pointwise to some $h^v$. Now let $T^*$ be the decision tree obtained by using $h^v$ as decision rule at $v$. We observe that $T^*=c$. Let $x \in X$. By Definition~\ref{def:pointwise_convergence}, for each node $v$ there exists $i^v_x$ such that $x \in h^v_i$ iff $x \in h^v$ for every $i \ge i^v_x$, where $h^v_i$ is the stump used at $v$ by $T^*_i$. By letting $i_x = \max_v i^v_x$ it follows that $x \in h^v_i$ iff $x \in h^v$ for every $i \ge i_x$ and all nodes $v$ simultaneously. Therefore all trees $T^*_i$ with $i \ge i_x$ send $x$ to the same leaf, and moreover that leaf remains the same if we use $h^v$ at $v$. Note also that, since $(T^*_i)_{i \in \N}$ is infinite, then we can assume that every leaf predicts the same label in all $T^*_i$ (since there is certainly an infinite subsequence that satisfies such a constraint). It follows that $(T^*_i)_{i \in \N}$ converges pointwise the tree $T^*$ that uses the limit stump $h^v$ at $v$. But the labeling of $(T^*_i)_{i \in \N}$ converges pointwise to $c$, too. We conclude that $T=T^*$. Finally, note that by construction $h^v \in \clospw(\scH)$, and thus by Lemma~\ref{lem:clos=clospw} $h^v \in\clos(\scH)$, for all $v$, hence $T^* \in \Alg(\clos(\scH))$. It follows that $c \in \Alg(\clos(\scH))$.
\end{proof}

%%%%%%%%appendix_gcm
\section{Remarks on the Graded Complexity Measure Results} \label{apx:gcm_remark}

In Section~\ref{sec:gen-rep} we demonstrated more general guarantees for any graded complexity measure $\Gamma$, given any domain $X$ and any hypothesis class $\scH$ over $X$.
Observe that, when $\scH$ is a non-VC class, item (3$'$) of Theorem~\ref{thm:trichotomy-gcm} states an upper bound on the $\Gamma$-complexity rate of order $\scO\bigl(\frac{1}{\eps^d}\bigr)$ for a constant $d\in\N$.
This bound is indeed larger compared to the previous guarantee of $\scO(\log(1/\eps))$ on the depth of $\scH$-based decision trees (Theorem~\ref{thm:trichotomy}) and it has to do with the generality of the definition of graded complexity measure.

Keeping this in mind, we remark that it is possible to recover the $\scO(\log(1/\eps))$ $\Gamma$-complexity rate bound under a stronger assumption on the graded complexity measure $\Gamma$.
In particular, it is sufficient for $\Gamma$ to satisfy
\begin{equation}
    \Gamma(f_1 \cup f_2) \le 1 + \max\{\Gamma(f_1), \Gamma(f_2)\} \qquad \forall f_1,f_2\in\Alg(\scH) \enspace.
\end{equation}
Note that this condition is satisfied when $\Gamma$ corresponds to the depth of $\scH$-based decision trees. For example, consider a similar representation of trees as in Equation~\eqref{eq:gcm_tree_to_alg} using directly $\scH$ for the decision rules of the internal nodes.

Thus, we can follow the same steps as in the proof of Theorem~\ref{thm:trichotomy-gcm} with a particular focus on the construction of $A$ from the decision tree $T$ in Equation~\eqref{eq:gcm_tree_to_alg}.
It immediately follows that $\Gamma(A_v^T) \le 3 + \Gamma(A_v) + \max\bigl\{\Gamma(A_u^T), \Gamma(A_w^T)\bigr\}$ for any internal node $v\notin\scL(T)$, where $u$ and $w$ are, respectively, the left and right child of $v$.
Now, let $\rho(z) \subseteq \scV(T)$ be the nodes along the path from the root of $T$ to the leaf $z\in\scL(T)$.
We can thus show that
\begin{equation}
    \Gamma(A) = \scO\Biggl(\max_{z\in\scL(T)} \sum_{v\in\rho(z)} (\Gamma(A_v)+1)\Biggr) = \scO\bigl((k+1)\cdot\depth(T)\bigr) = \scO\Bigl(\log\frac{1}{\eps}\Bigr) \enspace,
\end{equation}
where we used the fact that $T$ has $\depth(T) \le \frac{1}{2\gamma^2}\log\frac{1}{2\eps}$ and that $\Gamma(A_v)\le k$ for any internal node $v$ of $T$.

\end{document}